\documentclass[letterpaper, 10 pt, conference]{ieeeconf}
\usepackage{graphicx} % Required for inserting images
\IEEEoverridecommandlockouts
\usepackage{times}
\usepackage{multicol}
\usepackage[bookmarks=true]{hyperref}

\usepackage{hyperref}
\usepackage{graphicx}		% For \includegraphics
\usepackage{wrapfig}
\usepackage[format=plain,font=footnotesize,labelfont=bf,labelsep=period]{caption}
\usepackage[export]{adjustbox}
\usepackage{bm}
\usepackage{subcaption}
\usepackage{balance}
\usepackage{dsfont}
\usepackage{booktabs} 
\usepackage{amsmath}
\usepackage{mathrsfs}
\usepackage{amssymb}

\usepackage{amsthm}
\usepackage{amsthm}
\usepackage{amsfonts}
\usepackage{mathrsfs}  
\usepackage{mathtools}
 \usepackage{tikz} 
\usetikzlibrary{positioning,arrows.meta,shadows.blur,fit}
\usepackage{tikzducks}
\usepackage{duckuments}
\usepackage{mwe} % provides \includegraphics with example-image
\usepackage{mathtools}
\usepackage{relsize}
\usepackage{bbm}
\usepackage{cite}
\usepackage{algorithm}
\usepackage{algpseudocode}

\usepackage{pdfsync}
\usepackage[normalem]{ulem}
\usepackage{paralist}	% For enumerations with better titled numbers
\usepackage[space]{grffile} % Space in filenames

\usepackage[usenames,dvipsnames]{xcolor}
\usepackage{centernot}

\usepackage{graphicx}   % for \scalebox
\usepackage{environ}    % for \NewEnviron

\newtheorem{theorem}{Theorem}
\newtheorem*{theorem*}{Theorem}

\newtheorem*{lemma*}{Lemma}

\theoremstyle{definition}
\newtheorem{definition}{Definition}
\newtheorem{remark}{Remark}

\newtheorem*{proposition*}{Proposition}

\newcommand{\newsec}[1]{\vspace{0.1cm} \noindent \textbf{#1} }

\definecolor{darkblue}{RGB}{0,0,102}
\definecolor{lightblue}{RGB}{77,77,148}

\definecolor{gold}{RGB}{234, 170, 0}
\definecolor{metallic_gold}{RGB}{139, 111, 78}

\newcommand{\mb}[1]{\mathbf{ #1 }}

\DeclareMathOperator*{\argmin}{argmin}

\newcommand{\lmat }{\begin{bmatrix}}
\newcommand{\rmat}{\end{bmatrix}}

\usepackage{dsfont}
 
\usepackage{svg}
\usepackage{caption}

\makeatletter
\long\def\@makefntext#1{%
  \noindent
  #1}
\makeatother

\IEEEoverridecommandlockouts
\usepackage{mathtools}
\usepackage{xcolor}
\usepackage{url}

\newcommand{\re}{\mathbb{R}}
\newcommand{\bq}{\mathbf{q}}
\title{\LARGE \bf Safe-SAGE: \underline{S}ocial-Semantic \underline{A}daptive \underline{G}uidance for Safe \underline{E}ngagement through Laplace-Modulated Poisson Safety Functions
}
\author{Lizhi Yang\textsuperscript{1}, Ryan M. Bena\textsuperscript{1}, Meg Wilkinson\textsuperscript{1}, Gilbert Bahati\textsuperscript{1}, \\ Andy Navarro Brenes\textsuperscript{2}, Ryan K. Cosner\textsuperscript{2}, Aaron D. Ames\textsuperscript{1} \thanks{All authors affiliated with Caltech MCE\textsuperscript{1} and Tufts ME\textsuperscript{2}.}\thanks{ This research is supported in part by the Technology Innovation Institute (TII), BP p.l.c., and by The Dow Chemical Company project \#227027AW.}}
\date{December 2025}

\begin{document}

\maketitle
\begin{abstract}
Traditional safety-critical control methods, such as control barrier functions, suffer from semantic blindness, exhibiting the same behavior around obstacles regardless of contextual significance. 
This limitation leads to the uniform treatment of all obstacles, despite their differing semantic meanings. 
We present Safe-SAGE (Social-Semantic Adaptive Guidance for Safe Engagement), a unified framework that bridges the gap between high-level semantic understanding and low-level safety-critical control through a Poisson safety function (PSF) modulated using a Laplace guidance field. 
Our approach perceives the environment by fusing multi-sensor point clouds with vision-based instance segmentation and persistent object tracking to maintain up-to-date semantics beyond the camera's field of view. 
A multi-layer safety filter is then used to modulate system inputs to achieve safe navigation using this semantic understanding of the environment. This safety filter consists of both a model predictive control layer and a control barrier function layer. 
Both layers utilize the PSF and flux modulation of the guidance field to introduce varying levels of conservatism and multi-agent passing norms for different obstacles in the environment. 
Our framework enables legged robots to safely navigate semantically rich, dynamic environments with context-dependent safety margins.
%while maintaining  safety guarantees.
\end{abstract}
\section{Introduction}

Recent advances in legged locomotion have moved robots from controlled labs into semantically rich, human-centered environments, where safety becomes paramount. 
Established methods, for example artificial potential fields \cite{khatib1986real}, model predictive control (MPC) \cite{zhang2020optimization}, control barrier functions (CBFs) \cite{ames2016control}, and Hamilton-Jacobi (HJ) reachability \cite{herbert2017fastrack} provide rigorous guarantees via forward invariance of a ``safe set,'' yet their real-world utility hinges on how that safe set is constructed.

% To enable this safe deployment, there have been significant achievements in safety-critical control, ranging from artificial potential fields \cite{khatib1986real} and model predictive control (MPC) \cite{zhang2020optimization} to control barrier functions (CBFs) \cite{ames2016control} and Hamilton-Jacobi (HJ) reachability \cite{herbert2017fastrack}.
 %All of t
%Given a predefined safe region of the system's state space, these methods provide rigorous safety guarantees. 
% These methods can be used to guarantee that a system remains in a safe region of its state space. 
% Given a safe region of the system's state space, these methods provide rigorous safety guarantees for the robotic system. %In this work, we employ a combined MPC and CBF approach that benefits from the horizon-long performance optimization of the MPC controller and the real-time computational efficiency of the CBF. 
 % These methods provide rigorous mathematical safety guarantees defined as the forward invariance of a \emph{safe set}. %, i.e., a safe subset of the state space. 
% However, fundamental to the utility of these methods is the construction of this safe set.

Typically, these safe sets are either ``user-defined'' \cite{ames2016control,bansal_hj_2017}, learned from data \cite{qin2021learning,robey2020learning}, or constructed from environmental occupancy maps and geometric primitives \cite{TamasComposing23}. %s constructed with geometric primitives.
Such approaches are ``semantically blind'': a human and a chair of equal geometric volume yield identical safe sets and avoidance behavior. 
This forces controllers to be either universally conservative, degrading performance in clutter, or universally aggressive, risking safety failures. 
Prior state-dependent relaxations \cite{alan2021safe} tune conservatism but require hand-designed rules that cannot autonomously incorporate context.

Incorporating semantics makes safety a context-dependent property beyond simple collision avoidance, e.g. a hospital robot must distinguish walls that permit grazes from patients that demand a wide berth. 
While LLMs and vision-language models (VLMs) let robots reason about such semantics and social norms \cite{ahn2022can, ravichadran2026contextual}, a modality mismatch blocks their direct use for dynamic safety enforcement: these models run at low frequencies with high latency, whereas the safety filter and controller require near-real-time operation.

\begin{figure}
    \centering
    \includegraphics[width=\linewidth]{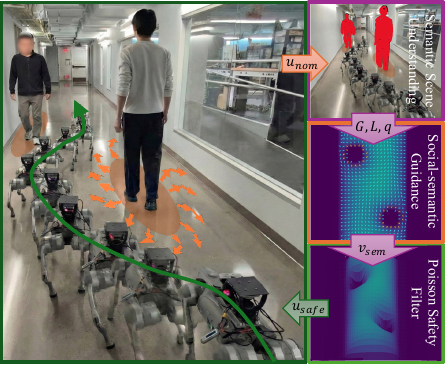}
    \caption{Safe-SAGE in action: A quadruped robot navigates a hallway with humans moving both towards and away from it. The robot successfully avoids collisions with humans and maintains social norms, passing on the left side of the humans.}
    \label{fig:head}
    % \vspace{-20pt}
\end{figure}

We propose a unified framework to bridge this \textit{neuro-symbolic gap}, extending Poisson safety functions (PSFs) \cite{bahati2025dynamic} and Laplace guidance fields (LGFs) \cite{bahati2025risk} into a semantics-aware safety layer between perception and control. 
It is distinguished by two mechanisms: 1) semantic flux modulation that adapts constraint activation near obstacle boundaries, and 2) a rotational LGF component that enables socially-compliant avoidance \cite{mavrogiannis2022social}, guiding safe behavior with variable conservativeness. 
For hardware enforcement, we adopt a layered approach \cite{yamaguchi2026layered}: an MPC safety filter \cite{bena2025geometry} plans over a finite horizon under linearized CBF constraints from the semantic LGF, while a real-time analytical filter with $\sigma$-scaling provides robustness to discrete-grid artifacts and immediate reactivity to dynamic hazards.

\subsection{Contributions}
The main contributions of this work are threefold:
1) a formulation for embedding social biases directly into the Laplace guidance field to enforce directional avoidance norms,
2) a formal analysis of the modified guidance field, and
3) a pipeline for integrating class-aware inflation and semantic flux modulation to generate obstacle-conforming safety and guidance fields from real-time perception.
% 3) An extension of the geometry-aware MPC framework in \cite{bena2025geometry} to include semantic dependency in the prediction horizon, coupled with the real-time safety filter utilizing the same semantics to enforce safety. 

\subsection{Related Work}
The synthesis of semantic reasoning with safety-critical control lies at the intersection of geometric safety and learning for semantics, utilizing Poisson safety functions.

\newsec{Geometric safety} enforces safety by constraining control inputs to render a safe set invariant \cite{ames2016control,bansal_hj_2017,borrelli_mpcBook_2017}. 
Geometry-based constraints must be converted to constraints on the dynamics via the planning horizon in MPC, the backward reachable tube in HJ methods \cite{herbert2017fastrack, bansal2021deepreach}, and high-relative-degree \cite{xiao2021high} or multi-layer synthesis \cite{cohen2024constructive} in CBFs.
% with extensions to handle high-relative-degree \cite{xiao2021high} or stochastic \cite{cosner2024bounding} systems.  
These methods do not distinguish between objects semantically; thus, may lead to universal conservatism or unsafe aggression.
% HJ reachability, on the other hand, guarantees safety by computing the backward reachable tube \cite{herbert2017fastrack,bansal2021deepreach} with variations of using the computed value function as the safety function in the CBF formulation \cite{choi2021robust,tonkens2022refining}. 
% While they are able calculate control invariant sets for complex systems, HJ reachability methods have limited real-time utility due to their large computational requirement and introduce extreme conservatism due to their game-based, worst-case uncertainty models. %, sometimes extending into latent space \cite{nakamura2025train} where the safety conditions are only understood empirically. 
% Recently, some CBF and reachability-based methods have begun incorporating semantic understandings into their safe set synthesis \cite{ravichadran2026contextual,feng2025words}, but the computational complexity of the LLM/VLM used to provide context as well as the reachability problem limits their utility in dynamic environments. 
Our proposed method extends geometric safety methods \cite{bena2025geometry,bahati2025dynamic} by incorporating semantic information using an onboard vision model, thus enabling the online synthesis of \textit{context-aware} safe sets. 

% thereby injecting more intelligence into the safety filter while leveraging the CBF approach, making the safety filter efficient and easy to understand.

\newsec{Safety from semantics} utilizes large vision/language models for robotic safety.
Language-conditioned planning \cite{ahn2022can, huang2023voxposer} grounds semantic instructions to robotic capabilities.
Other works attempt to learn barrier functions directly from data \cite{robey2020learning, long2021learning} or for multi-agent systems \cite{qin2021learning}.
With the advancement of LLMs and VLMs, works such as \cite{brunke2025semantically, hu2025vlsa, ravichadran2026contextual} generate barrier conditions from visual-language queries.
\cite{seo2025uncertainty,  nakamura2025train} extend the notion of safety to hard-to-describe safety conditions by leveraging HJ reachability and latent-space safe sets.
These methods operate largely at the planning layer, ensuring semantic feasibility but often fail to provide execution-level guarantees against dynamic disturbances.
Our approach, through the combination of the MPC and the real-time safety filter, safeguards the robotic system against such scenarios. 

\newsec{Harmonic potentials} \cite{connolly1990path} guarantee local-minima-free navigation, with \cite{doeser2020invariant} investigating invariant sets for integrators using similar principles.
More recent work \cite{bahati2025dynamic} established the Poisson safety function framework, which utilizes as an MPC constraint \cite{bena2025geometry} and \cite{bahati2025risk} further expands it to support context-based risk assignment.
Our work is the first to introduce semantic flux modulation into this architecture, enabling semantic-aware social and safety compliance.

\newsec{Social navigation} requires the robot to adhere to social norms when navigating in human environments \cite{sisbot2007human}.
Early works utilized the social force model \cite{helbing1995social} to simulate pedestrian dynamics.
More recent approaches leverage deep reinforcement learning \cite{chen2019crowd} or topological invariants \cite{mavrogiannis2022social} to learn cooperative behaviors.
Gaussian processes are also used to address the ``freezing robot" problem in dense crowds \cite{trautman2010unfreezing}, and generative models have been used to generate socially plausible trajectories \cite{zhou2025socialtraj}.
However, these methods lack the rigorous safety guarantees of CBFs or require heavy computational resources, which is not ideal for real-time operation. 
Our flux-modulated guidance field naturally embeds social behaviors directly into the real-time safety filter, ensuring both social compliance and rigorous safety.
% \vspace{-14pt}
% \vspace{-2pt}
\section{Background}
\label{sec:background}
% \vspace{-8pt}
% \vspace{-2pt}
\newsec{Reduced-Order Models}
 provide a means of representing a high-dimensional robotic system (with dynamics ${\Dot{\boldsymbol{\zeta}}}=\phi(\boldsymbol{\zeta},\mathbf{u})$, $\boldsymbol{\zeta}\in \mathbb{R}^n$, and $\mathbf{u}\in \mathbb{R}^m$) as a reduced-order system by considering a reduced-order state $\mathbf{q}\in \mathbb{R}^{n_q}$, $n_q < n$, where $\mathbf{q}$ can capture important bulk behavior such as the position of the robot centroid.
 Thus, given the projection $\mathbf{p}: \mathbb{R}^n\rightarrow\mathbb{R}^{n_q}$ that projects the full-order state to the reduced-order state, a % locally Lipschitz continuous 
 feedback control law $\mathbf{u}=\mathbf{k}(\mathbf{q})\in \mathbb{R}^{m_q}$, $\mathbf{k}:\mathbb{R}^{n_q}\rightarrow\mathbb{R}^{m_q}$, and a control interface $\kappa(\boldsymbol{\zeta},\mathbf{u})$ that lifts the reduced-order input $\mathbf{u}$ to the full-order input $\mathbf{u}_{\text{full}}$, the reduced-order state $\mathbf{q}$ satisfies the following \cite{cohen2023safe,cohen2025safety}:
 \begin{align}
    \dot{\mathbf{q}} = \frac{\partial \mathbf{p}}{\partial \boldsymbol{\zeta}} \, \mathbf{\phi}\big(\boldsymbol{\zeta}, \, \mathbf{\kappa}(\boldsymbol{\zeta}, \mathbf{u})\big) \nonumber &\approx \mathbf{f}(\mathbf{q}) + \mathbf{g}(\mathbf{q}) \mathbf{u} 
    \label{eq:cont_reduced_order_sys}
    \\
     &= \mathbf{f}(\mathbf{q}) + \mathbf{g}(\mathbf{q}) \mathbf{k}(\mathbf{q}).
\end{align}
 
\newsec{Safety Filters} enforce safety constraints on the reduced-order system, which can be characterized using a \emph{safe set} $\mathcal{C}$:
\vspace{-5pt}
\begin{align}
\mathcal{C} &:= \{ \mathbf{q} \in \mathbb{R}^{n_q} \mid h(\mathbf{q}) \ge 0 \}, \nonumber\\
\partial \mathcal{C} &:= \{ \mathbf{q} \in \mathbb{R}^{n_q} \mid h(\mathbf{q}) = 0 \}, \label{eq:safesetdef}
\\
\operatorname{int}(\mathcal{C}) &:= \{ \mathbf{q} \in \mathbb{R}^{n_q} \mid h(\mathbf{q}) > 0 \},\nonumber
\end{align}
where the 0-superlevel set of a continuously differentiable function $h:\mathbb{R}^{n_q}\rightarrow\mathbb{R}$ represents the safe subset of the reduced-order state space, $\mathbb{R}^{n_q}$.
Thus, the aim of a safety filter is to enforce the forward invariance of $\mathcal{C}$ \cite{ames2019control}.
\begin{figure*}[htbp]
 \centering
    \includegraphics[width=\linewidth]{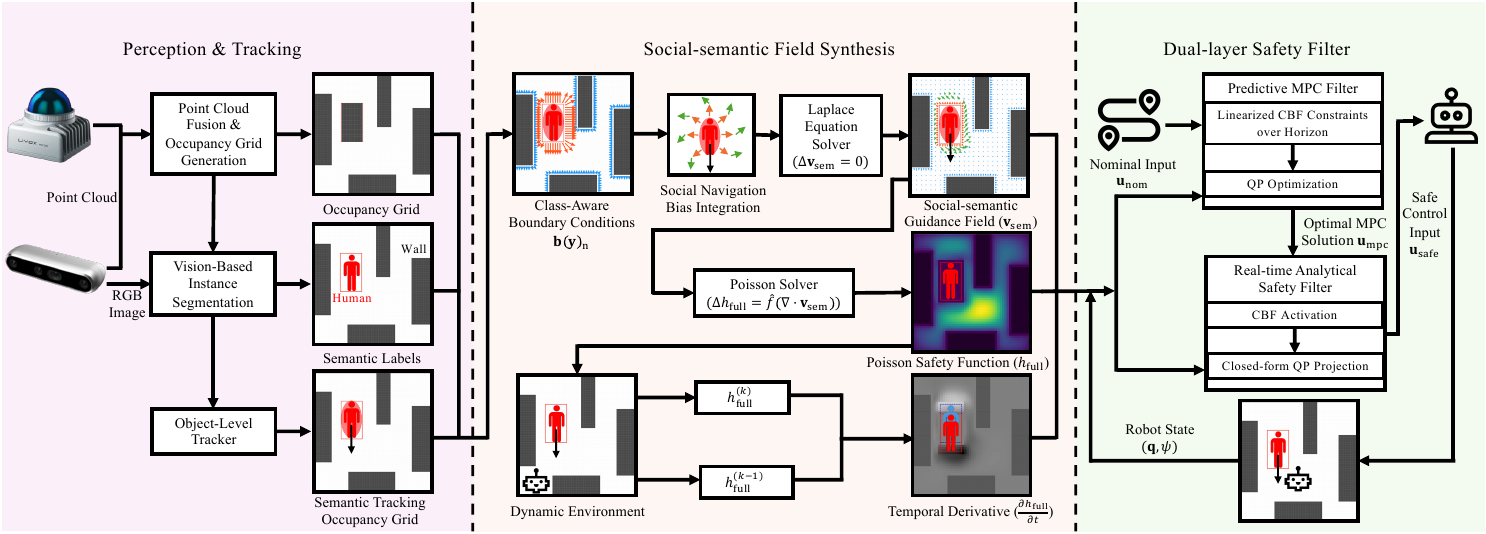}
    \caption{System Architecture: The robot takes in multi-sensor point clouds and RGB images from the camera, performs semantic segmentation and object tracking to build a semantic occupancy grid, and then uses it to generate a social-semantic guidance field and Poisson safety function, then apply it in both real-time and predictive safety filters to ensure safety and social compliance.}
    \label{fig:system}    
    % \vspace{5pt}
\end{figure*}
\begin{definition}[\hspace{-0.05em}\cite{ames2019control}]
A continuously differentiable function $h : \mathbb{R}^{n_q} \rightarrow \mathbb{R}$, satisfying $\nabla h(\mathbf{q})\neq0$ when $h(\mathbf{q})=0$, is a
\emph{control barrier function} (CBF) for \eqref{eq:cont_reduced_order_sys} on $\mathcal{C}$ if there exists
$\gamma \in \mathcal{K}_\infty^e$ such that for all $\mathbf{q} \in \mathbb{R}^{n_q}$, the
following holds:%\pagebreak
\begin{equation}
\sup_{\mathbf{u} \in \mathbb{R}^{m_q}}
\left\{
\underbrace{\nabla h(\mathbf{q}) \mathbf{f}(\mathbf{q})}_{L_\mathbf{f} h(\mathbf{q})}
+
\underbrace{\nabla h(\mathbf{q}) \mathbf{g}(\mathbf{q})}_{L_\mathbf{g} h(\mathbf{q})}\mathbf{u}
\right\}
\ge -\gamma\bigl(h(\mathbf{q})\bigr),
\end{equation}
where $\nabla h$ denotes the gradient of $h$.
\end{definition}
\noindent Inputs $\mb{u}$ satisfying this constraint render $\mathcal{C}$ control invariant \cite[Thm. 2]{ames2019control}, and the following CBF-QP controller ensures safety by minimally modifying a nominal action $\mathbf{k}_{\textup{nom}}$:
\begin{equation}
\label{eq:cbf-qp}
\begin{aligned}
\mathbf{k}_{\mathrm{safe}}(\mathbf{q})
&= \argmin_{\mathbf{u} \in \mathbb{R}^{m_q}}
\;\left\| \mathbf{u} - \mathbf{k}_{\mathrm{nom}}(\mathbf{q}) \right\|_2^2
 \\
\text{s.t.}\quad
& L_\mathbf{f} h(\mathbf{q}) + L_\mathbf{g} h(\mathbf{q})\,\mathbf{u}
\;\ge\; -\gamma\!\left(h(\mathbf{q})\right).
\end{aligned}
\end{equation}
Since real-world sensing and actuation use a zero-order hold, we also employ the discrete-time CBF for systems of the form:
\begin{equation}
    \mathbf{q}_{k+1} = \mathbf{F}(\mathbf{q}_k, \mathbf{u}_k),
    \label{eq:rom}
\end{equation}
\begin{definition}[\hspace{-0.05em}\cite{agrawal_dtcbf_2017}]
A function $h : \mathbb{R}^{n_q} \rightarrow \mathbb{R}$ is a
\emph{discrete-time CBF} (DTCBF) for \eqref{eq:rom} on $\mathcal{C}$ if, for some $\rho \in [0, 1]$ and each $\mathbf{q}\in \mathcal{C}$, there exists a $\mathbf{u}\in \mathbb{R}^{m_q}$ such that: 
\begin{equation}
    h(\mathbf{F}(\mathbf{q},\mathbf{u})) \geq \rho h(\mathbf{q}),
\end{equation}
\end{definition}
The associated discrete-time safety filter is formulated as
\begin{equation}
\label{eq:cbf-qp-discrete}
\begin{aligned}
\mathbf{k}_{\mathrm{safe}}(\mathbf{q})
&= \argmin_{\mathbf{u} \in \mathbb{R}^{m_q}}
\;\left\| \mathbf{u} - \mathbf{k}_{\mathrm{nom}}(\mathbf{q}) \right\|_2^2
 \\
\text{s.t.}\quad
& h(\mathbf{F}(\mathbf{q},\mathbf{u})) \geq \rho h(\mathbf{q}).
\end{aligned}
\end{equation}
\noindent While this may no longer be convex, \cite{taylor2022safety} showed that safety can be preserved under certain convexifying approximations. 

\newsec{Poisson Safety Functions} are a class of functions that can be used to characterize safety with respect to a spatial environment $\mathbf{q}=(x,y,z)\in \mathbb{R}^3$ \cite{bahati2025dynamic,bahati2025risk}.
PSFs represent the safety of any environment with a smooth, open, bounded, connected free space $\Omega$, with $\partial \Omega = \bigcup_{i=1}^{n_o} \partial \Gamma_i$ the surfaces of the $n_o$ occupied regions $\Gamma_i$.
Given an occupancy map (numerical safe set), we generate a PSF by solving a Dirichlet boundary value problem for Poisson's equation \eqref{eq:safesetdef}:
\begin{equation}
\begin{cases}
\Delta h_0(\mathbf{q}) = \hat{f}(\mathbf{q}), & \mathbf{q}\in \Omega, \\
h_0(\mathbf{q}) = 0, & \mathbf{q} \in \partial \Omega,
\end{cases}
\label{eq:poisson_dirichlet}
\end{equation}
where $\Delta = \frac{\partial}{\partial \mathbf{q}} \cdot \frac{\partial}{\partial \mathbf{q}}$ denotes the Laplacian operator, and
$\hat{f} : \Omega \rightarrow \mathbb{R}_{<0}$ is a prescribed forcing term.
It is noted that a smooth forcing function
$\hat{f} \in C^{\infty}(\overline{\Omega})$ gives a smooth
solution $h_0 \in C^{\infty}(\overline{\Omega})$ to
\eqref{eq:poisson_dirichlet} under regularity assumptions on $\Omega$\cite{gilbarg1977elliptic}.
The resultant 0-superlevel set of $h_0$ implicitly defines a safe set $\mathcal{C}$ such that $\Omega = \operatorname{int}(\mathcal{C})$ and  $\partial \mathcal{C} = \partial \Omega$ \cite{bahati2025dynamic}. Under suitable assumptions, the PSF is a CBF, and this constructive method enables the design of safety filters that generate admissible control inputs for systems \eqref{eq:cont_reduced_order_sys} and \eqref{eq:rom}.

\newsec{Guidance-field-based Safety-critical Control} is used to enable spatially varying conservatism around obstacle boundaries using Laplace Guidance Fields (LGFs) by explicitly prescribing safe directionality independent of the safety function's gradient.
Let $\hat{\mathbf{n}}(\mb{q})$ be the outward unit normal for the safe set boundary at $\mb{q}$, where $b(\mathbf{q})$ is the boundary value. A guidance field $\mathbf{v} : \Omega \to \mathbb{R}^3$ satisfies the boundary flux condition
\begin{equation}
\mathbf{v}(\mathbf{q}) = b(\mathbf{q})\,\hat{\mathbf{n}}(\mathbf{q}),
\quad b(\mathbf{q}) < 0,
\quad \mathbf{q} \in \partial\Omega.
\end{equation}
We extend the field smoothly into $\Omega$ by solving the vector Laplace Dirichlet boundary value problem
\begin{equation}
\begin{cases}
\Delta v_i(\mathbf{q}) = 0, & \mathbf{q} \in \Omega, \\
v_i(\mathbf{q}) = b(\mathbf{q})\,\hat n_i(\mathbf{q}), & \mathbf{q} \in \partial\Omega,
\end{cases}
\qquad i \in \{x,y,z\},
\end{equation}
where $v_i\in\mathbb{R}$ represents the components of $\mathbf{v}$.

Although $\mathbf{v}$ is generally non-conservative and not equal to $\nabla h$,
it satisfies\footnote{$^1$ $\parallel$ denotes that the vectors are parallel.}  $\mathbf{v} \parallel \hat{\mathbf{n}}$ on $\partial\Omega$. By
Hopf's Lemma \cite{protter2012maximum}, the outward normal satisfies\footnote{$^2$ $\propto$ denotes proportionality.}
$\hat{\mathbf{n}} \propto \nabla h$ on $\partial\Omega$, enabling the generalized
safety constraint
\begin{equation}
\mathbf{v}(\mathbf{q})^\top \mathbf{k}(\mathbf{q}) \ge -\gamma h(\mathbf{q}),
\end{equation}
which guarantees safety for first-order systems \cite[Prop. 1]{bahati2025risk}.

\section{Method}
We describe our social-semantic safety framework built around LGF-based PSFs, beginning with the system architecture and then detailing each component: perception and tracking, field synthesis, and the dual safety filter.
% \vspace{-1pt}
\subsection{System Architecture}
% TODO maybe talk about actual sensors in the exp section
The proposed system operates as a layered safety architecture positioned between perception and control, as illustrated in Figure~\ref{fig:system}. 
We first fuse multi-sensor point clouds into a robot-centric occupancy grid and use a vision-based segmentation network to identify human instances.
We then deploy an object-level tracker for persistent human identification outside the camera's field of view.
Finally, the resulting semantic occupancy grid feeds into a two-stage field synthesis module, where we solve Laplace's equation with class-aware boundary conditions to construct a semantic guidance field $\mathbf{v_{\text{sem}}}: \mathbb{R}^2 \to \mathbb{R}^2$, and construct a forcing function for solving Poisson's equation to arrive at a safety function $h_{\text{full}}:  \mathbb{R}^2 \times \mathbb{S}^1 \to \mathbb{R}$. 
Both serve as inputs to a dual safety filter architecture operating on a reduced-order model with full state $\boldsymbol{\zeta} = (\mathbf{q}, \psi)$, where $\mathbf{q} = (x, y)$ is the 2D position and $\psi$ is the heading. 
The safety module consists of an analytical filter providing immediate reactive corrections and an MPC filter enforcing predictive CBF constraints. 
This dual-rate structure combines reflexive responsiveness for dynamic environments with anticipatory trajectory shaping for smooth, socially-compliant navigation (Alg.~\ref{alg:semantic_safety}).
{
\setlength{\textfloatsep}{6pt}
\begin{algorithm}[t]
\caption{Safe-SAGE}
\label{alg:semantic_safety}
\begin{algorithmic}[1]
\Require Occupancy grid $\mathcal{G}$, class map $\mathcal{L}$, robot state $(x, y, \psi)$, nominal command $\mathbf{u}_{\text{nom}}$
\Ensure Safe command $\mathbf{u}_{\text{safe}}$

\Statex \textbf{Poisson Safety Function Construction:}
\State Extract LiDAR clusters via connected components
\State Update human tracker with clusters and YOLO labels
\State Label occupied cells using tracked human positions
% \State Inflate $\mathcal{G}$ with class-specific kernels
\State Compute boundary normals $\hat{\mathbf{n}}$ and class-aware flux $b(\mathbf{q})$, construct $\mathbf{v_{\text{sem}}}$ with tangent biasing 
% \For{each yaw slice $q \in \{1, \ldots, Q\}$} \textbf{in parallel}
\State Solve Laplace equation $\Delta \mathbf{v}_{{\text{sem}}} = 0$ to construct $\mathbf{v}_{{\text{sem}}}$
\State Compute forcing function $\hat{f}(\nabla \cdot \mathbf{v_{\text{sem}}})$
\State Solve Poisson equation $\Delta h_{\text{full}} = \hat{f}$ for $h_{\text{full}}$
% \EndFor
\State Update temporal derivative $\partial h/\partial t$ via motion-compensated difference 
% \Statex \textbf{MPC Loop:}
% % \If{MPC enabled and new grid available}
%     \For{$i = 1$ to MAX\_ITERATIONS}
%         \State Update QP cost with $\mathbf{u}_{\text{nom}}$
%         \State Linearize CBF constraints using $\mathbf{v}$ and $\partial h / \partial t$
%         \State Solve QP via OSQP
%         \If{cost residual $< 1.0$} \textbf{break} \EndIf
%     \EndFor
\State $\mathbf{u}_{\text{mpc}} \gets$ MPC solution $\mathbf{u}_0^*$
% \EndIf

% \Statex \textbf{State-Rate Loop (100~Hz):}
% \State Interpolate $h$, $\mathbf{v_{\text{sem}}}$, $\partial h / \partial t$ at $(x, y, \psi)$ via trilinear interpolation
% \State Compute activation $a \gets \gamma h + \mathbf{v_{\text{sem}}}^\top \mathbf{u}_{\text{mpc}} + \sigma\frac{\partial h}{\partial t}$%$ - \text{ISSf}$
% \State Compute $\beta \gets \mathbf{v_{\text{sem}}}^\top P_u^{-1} \mathbf{v_{\text{sem}}}$
% \State Compute $\lambda \gets \frac{-a + \sqrt{a^2 + \sigma b^2}}{2b}$ 
% \State $\mathbf{u_{\text{safe}}} \gets \mathbf{u}_{\text{mpc}} + \frac{-a + \sqrt{a^2 + \beta^2}}{2\beta} P_u^{-1} \mathbf{v_{\text{sem}}}$
\State $\mathbf{u_{\text{safe}}} \gets$ real-time semantics-aware safety filter \eqref{eq:cbf-qp-guidance} %explicit solution

\State \Return $\mathbf{u_\text{safe}}$
\end{algorithmic}
\end{algorithm}
% \vspace{-5pt}
}
\subsection{Semantic Environment Representation}
\label{sec:env_rep}

We build a robot-centric occupancy grid by fusing point clouds from multiple sensors in the robot body frame.
We mask out the robot itself using a simple hyper-elliptical filter and maintain cell occupancy using exponential decay, along with a Gaussian kernel to spatially smooth detections.
An instance segmentation network processes the RGB stream, provides the semantic labels, and projects them into the robot body frame to produce a sparse class map.
However, this is insufficient for persistent human tracking, particularly when individuals move outside the camera's limited field of view.
To address this, we deploy an object-level tracker that associates LiDAR clusters with semantic labels over time. 
We first use connected component analysis \cite{bolelli2019spaghetti} to extract spatial clusters from the occupancy grid and match them to existing detections via a greedy nearest-neighbor association within a tunable gating radius.
The labeled clusters then update their position and velocity estimates using exponential smoothing and decay according to new detections and tracked clusters. 
Labeled clusters that lose association for more than the timeout threshold are pruned.
The downstream safety filters can thus maintain up-to-date semantics even without fresh visual confirmation.

\subsection{Laplace Guidance Field}
% Importantly, with this $B_r$ buffering $\partial \Omega$ is still the safe set boundary. $\partial \Omega_r$ is instead used to introduce social avoidance behavior and does not increase conservatism.

% \subsection{Laplace Guidance Field}
% With the configuration space occupancy map, we prescribe outward-pointing boundary conditions on the obstacle boundaries $\partial\Omega$ with class-dependent magnitudes:
% \begin{equation}
% \mathbf{v_{\text{sem}}}(y) = b(y)\,\hat{\mathbf{n}}(y), \quad y\in\partial\Omega,
% \end{equation}
% where $\hat{\mathbf{n}}$ is the outward unit normal and $b(y)$ is the boundary flux magnitude, with higher flux incentivizing more conservative safety behavior.
%
% The key idea is to construct a guidance vector field that encodes a socially prescribed flow direction.% in the frees space immediately surrounding each obstacle. 
%
% This is achieved by solving a vector Laplace equation with tangential boundary conditions in the obstacle's immediate free space—encoding the desired flow direction—and normal boundary conditions on the obstacle surface—encoding repulsive safety gradients.
% %
% Laplace's equation smoothly interpolates between these, yielding a field that transitions from social compliance far from obstacles to safety enforcement (i.e., collision avoidance) near obstacles. 
To account for the robot's orientation-dependent geometry, we
parameterize the domain with the robot orientation $\psi \mapsto \Omega({\psi})$ as introduced in
% introduce the 2D attitude-dependent free space $\Omega(\psi)$ following 
\cite{bena2025geometry}.
We construct the social-semantic guidance field $\mathbf{v}_\textup{sem} = (v_x, v_y)$ to encode class-dependent social navigation norms by prescribing outward-normal repulsion on obstacle boundaries $\partial \Omega({\psi})$ that have class-dependent magnitude $b$ and a tangential bias on an internal Dirichlet interface $\partial \Omega_r({\psi})$, solved via a vector Laplace equation over free space.
Specifically, let $\Omega({\psi}) \subset \mathbb{R}^2$ be an open, connected domain encoding free space, with a smooth boundary $\partial \Omega({\psi})$ encoding obstacle surfaces.
To define $\partial \Omega_r({\psi})$, we buffer the obstacle boundaries via the Pontryagin difference \cite{bena2025geometry}:
\begin{align}\label{eq: buffered safe set}
    \overline{\Omega}_r(\psi) = \overline{\Omega}(\psi) \ominus B_r, %(\psi),
\end{align}
where $B_r = \left\{ \mathbf{q} \in \mathbb{R}^2 \;\middle|\; \|\mathbf{q}\| < r \right\}$ denotes the open ball of radius $r>0$. We assume $r$ is chosen small enough that the buffered regions surrounding distinct obstacles do not overlap. Under the smoothness assumption on $\partial\Omega({\psi})$ and for such $r$, the buffered boundary $\partial \Omega_r(\psi)$ is smooth and admits well-defined outward unit normals $\hat{\mathbf{n}}_r$ and unit tangents $\hat{\tau}_r$.

On $\partial \Omega_r({\psi})$, tangential boundary conditions encode the desired social flow direction, where the sign convention on $\hat{\tau}$ determines the passing direction (e.g., pass-on-the-right). On $\partial \Omega(\psi)$, outward-normal conditions encode repulsion. Specifically, for each component $i \in \{x, y\}$:
\begin{equation}\label{eq: social laplace}
    \begin{cases}
        \Delta v_i(\mathbf{q}, \psi) = 0, & \mathbf{q} \in \Omega(\psi) \setminus \partial \Omega_r(\psi), \\
        v_i(\mathbf{q}, \psi) = \lambda(\mathbf{q},  \psi) \, \hat{\tau}_i(\mathbf{q,  \psi}), & \mathbf{q} \in \partial \Omega_r(\psi), \\
        v_i(\mathbf{q}, \psi) = b(\mathbf{q}, \psi) \, \hat{n}_i(\mathbf{q},  \psi), & \mathbf{q} \in \partial \Omega(\psi),
    \end{cases}
\end{equation}
where $\lambda(\mathbf{q}, \psi) < 0$ controls the social flow magnitude and $b(\mathbf{q},\psi) < 0$ the repulsion magnitude. Since $\partial \Omega({\psi})$ and $\partial \Omega_r({\psi})$ are smooth and the boundary data is smooth, the Dirichlet problem is well-posed by classical elliptic theory \cite{gilbarg1977elliptic}. The solution is smooth on each subdomain $\mathcal{A}_r({\psi}) = \Omega_r({\psi}) \setminus \overline{\Omega}({\psi})$ and $\Omega_r({\psi})$, continuous across $\partial \Omega_r({\psi})$ by construction, and Lipschitz on $\overline{\Omega}({\psi})$.
These class-dependent knobs shape behavior intuitively: a more negative $b$ yields stronger repulsion, hence a wider margin and earlier activation, and is assigned by class safety-criticality (a larger margin for humans than static objects). The tangential bias $\lambda$ sets the passing side via its sign and, via its magnitude, how strongly social flow dominates away from the boundary; acting off the obstacle surface, it shapes social behavior without weakening the boundary safety constraint.
We hand-tune these per class; principled or learned assignment is left to future work.

\begin{remark}
    As established in \cite{bahati2025dynamic,bahati2025risk}, Laplace guidance fields are non-conservative (i.e., they have a nonzero curl). This is the property that precisely enables the encoding of rotational social flow patterns on the closed curve $\partial \Omega_r({\psi})$.
\end{remark}

\subsection{Poisson Safety Function Construction}

Utilizing the guidance field, we construct a scalar safety function $h_{\text{full}}:\mathbb{R}^2 \times \mathbb{S}^1 \to \mathbb{R}$ over a 3D configuration space comprising position and yaw orientation $\boldsymbol{\zeta} = (\mathbf{q}, \psi)$.
% To account for the robot's orientation-dependent geometry, we
% parameterize the domain with the robot orientation $\psi \mapsto \Omega({\psi})$ as introduced in
% % introduce the 2D attitude-dependent free space $\Omega(\psi)$ following 
% \cite{bena2025geometry}.
% that constructs a different $\Omega$
% for each slice in the yaw axis, 
% The parametrization enables us to define a new lifted domain in
% a higher-dimensional space accounting for the yaw axis: 
% to create the new non-cylindrical domain \cite{bena2025geometry}:
% \begin{equation} \label{eq: lifted domain}
%     \tilde{\Omega} = \bigcup_{\psi \in \mathbb{S}^1} \Omega({\psi}) \times \{\psi\} \subset \mathbb{R}^2 \times \mathbb{S}^1.
% \end{equation}
We can solve for the rotationally-dependent Poisson safety function:
\begin{equation}
    \begin{cases} 
        \Delta h_{\text{full}}(\bq, \psi) = \hat{f}(\bq, \psi), & \forall\bq \in {\Omega}(\psi), \\ 
        h_{\text{full}}(\bq, \psi) = 0, & \forall\boldsymbol{\bq} \in \partial {\Omega}(\psi),
    \end{cases} 
\end{equation}
% where the Laplace operator $\Delta_{\mathbf{q}}$ is taken with respect to the 2D position $\mathbf{q} = (x, y)$. 
% Because the original safe set only defines safety with respect to the translational output $\mathbf{q}$, the yaw orientation $\psi$ does not need to satisfy a boundary condition along the rotational dimension and thus does not appear in the Laplacian. 
We use a smooth negative forcing function $\hat{f}(\bq,{\psi})$ following \cite{bena2025geometry}.
The resulting $h_{\text{full}}$ is smooth, strictly positive in free space according to the strong maximum principle, and zero on boundaries. 
Importantly, the guidance field's class-aware boundary conditions propagate into the forcing function, causing $h_{\text{full}}$ to change more quickly near semantically-critical boundaries compared to less critical obstacles.
% {\color{blue}Since the Laplacian and boundary conditions act only on $\mathbf{q}$ for each fixed $\psi$, we can state the following forward invariance result for each 2D yaw slice $h$ independently on the safe set $\mathcal{C}({\psi}) = \{\mathbf{q} \in \overline{\Omega}({\psi}) ~|~ h(\mathbf{q}) \geq 0\}$.}

% \textcolor{blue}{TODO: Cut theorem}
{\begin{theorem}[Forward Invariance]\label{prop:forward_invariance}
    For a fixed $\psi$, consider the single-integrator system $\dot{\mathbf{q}} = \mathbf{k}(\mathbf{q})$ on $\overline{\Omega}({\psi})$ and the safe set $\mathcal{C}({\psi})$ defined as the $0$-superlevel set of a continuously differentiable function $h_\psi:\overline{\Omega}({\psi}) \to \re$.
 Let $\mathbf{v} : \overline{\Omega}({\psi}) \to \mathbb{R}^2$ be a Lipschitz continuous vector field satisfying \eqref{eq: social laplace}. Then for any locally Lipschitz controller $\mathbf{k} : \overline{\Omega}({\psi}) \to \mathbb{R}^2$ satisfying:
    \begin{equation}\label{eq: risk aware constraint}
        \mathbf{v}(\mathbf{q}) \cdot \mathbf{k}(\mathbf{q}) \geq -\gamma\, h_{\psi}(\mathbf{q}), \quad \forall\, \mathbf{q} \in \mathcal{C}({\psi}),
    \end{equation}
    for some $\gamma > 0$, the set $\mathcal{C}({\psi})$ is rendered forward invariant.
\end{theorem}

\begin{proof}
    On $\partial \mathcal{C}({\psi}) = \partial \Omega({\psi})$, we have $h_{\psi}(\mathbf{q}) = 0$, so the constraint reduces to $\mathbf{v}(\mathbf{q}) \cdot \mathbf{k}(\mathbf{q}) \geq 0$. Since $\mathbf{v}(\mathbf{q}) \parallel \nabla h_{\psi}(\mathbf{q})$ on $\partial \Omega({\psi})$, that is, $\mathbf{v}$ is parallel to $\nabla h_{\psi}$ in the same direction, we obtain (dropping dependence on $\mathbf{q}$ for brevity):
    \begin{equation}
        \mathbf{v} \cdot \mathbf{k} = \frac{\|\mathbf{v}\|}{\|\nabla h_{\psi}\|} \nabla h_{\psi} \cdot \mathbf{k} = \frac{\|\mathbf{v}\|}{\|\nabla h_{\psi}\|} \dot{h}_{\psi} \geq 0.
    \end{equation}
    Since $\|\mathbf{v}\|/\|\nabla h_{\psi}\| > 0$, this implies $\dot{h}_{\psi}(\mathbf{q}) \geq 0$ for all $\mathbf{q} \in \partial \mathcal{C}({\psi})$. Forward invariance of $\mathcal{C}({\psi})$ then follows from Nagumo's theorem \cite{blanchini1999set}.
\end{proof}}
This guarantee holds up to the limits of perception: as the safe set is built from the semantic occupancy grid, segmentation and occupancy errors are the practical bound on safety enforcement.

\begin{figure}[t]
    \centering
    \begin{minipage}{\linewidth}
        \centering
        \includegraphics[width=\linewidth]{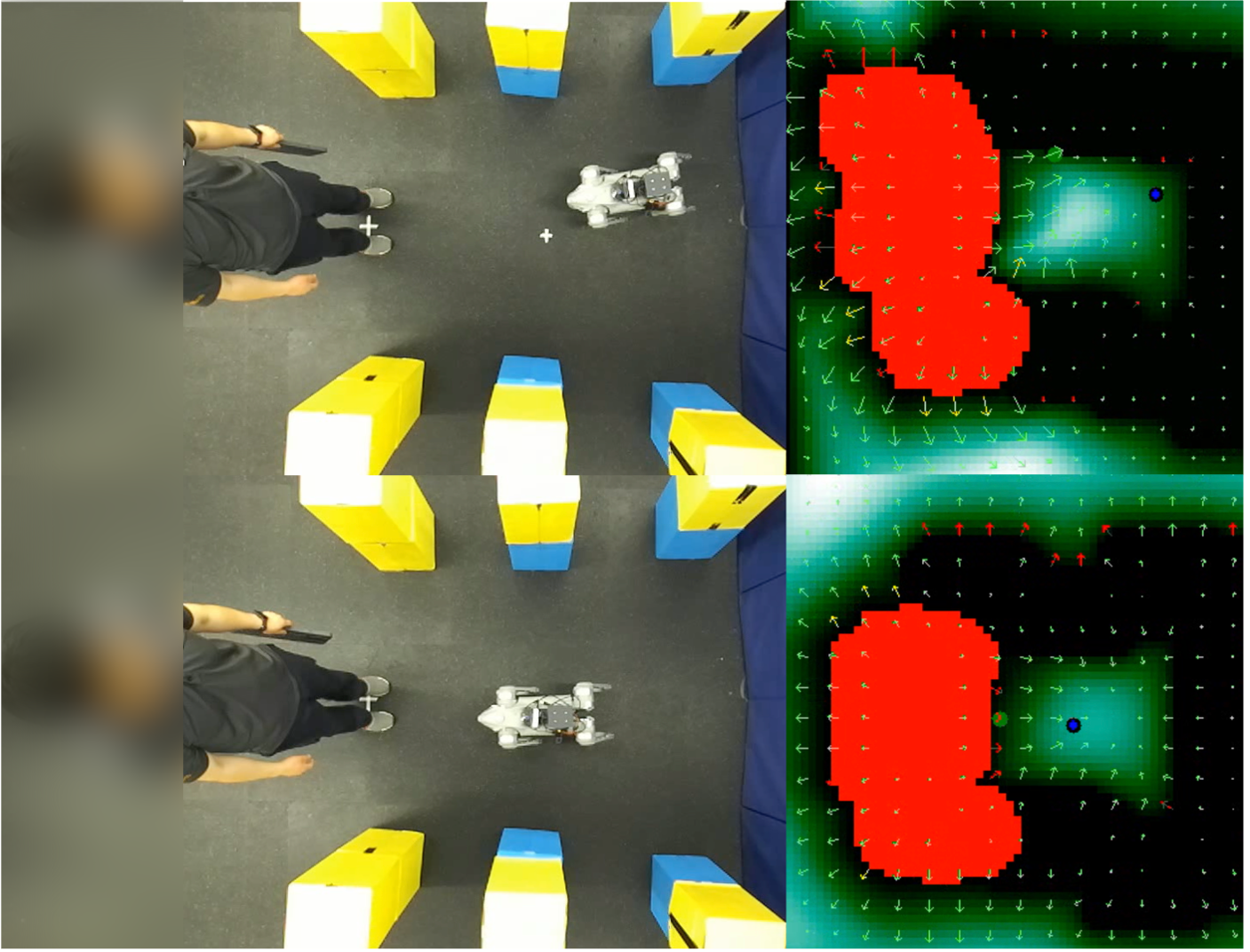}
        \caption{An example of the biased margin induced by our proposed method. With our method enabled ($b_{\text{human}}(\mathbf{q})=-1.7$ and $b_{\text{objects}}(\mathbf{q})=-0.5$), the robot keeps a wider margin to the human (whose occupies states are shown in red) than the walls (in black). While without it ($b_{\text{human}}(\mathbf{q})=-1.0$ and $b_{\text{objects}}(\mathbf{q})=-1.0$), it keeps the same margin.}
        \label{fig:weight_benchmark}
    \end{minipage}

    \vspace{1em}

    \begin{minipage}{\linewidth}
        \centering
        \captionof{table}{Safety Metrics: Human-robot margin is defined as the distance from the human to the center of the enclosed space formed by 3 walls and one human; max lateral offset is defined as the maximum distance to one side of the robot for it still to pass on other side of the human; in both larger is better.}
        \label{tab:safety-ablation}
        \resizebox{\linewidth}{!}{%
            \begin{tabular}{|l|c|c|}
                \hline
                Metric & Proposed Method & Baseline \\
                \hline
                Human-robot Margin (m) & $0.318 \pm 0.0774$ & $-0.008 \pm 0.0625$ \\
                Max Lateral Offset (m) & 0.75 & -0.1 \\
                \hline
            \end{tabular}
        }
    \end{minipage}
    \vspace{-6pt}
\end{figure}
\begin{figure*}[t]
    \centering
    % Add [t] for top alignment
    \begin{minipage}[t]{0.48\textwidth} 
        \vspace{0pt} % Creates a top-anchor point
        \centering
        \includegraphics[width=\linewidth]{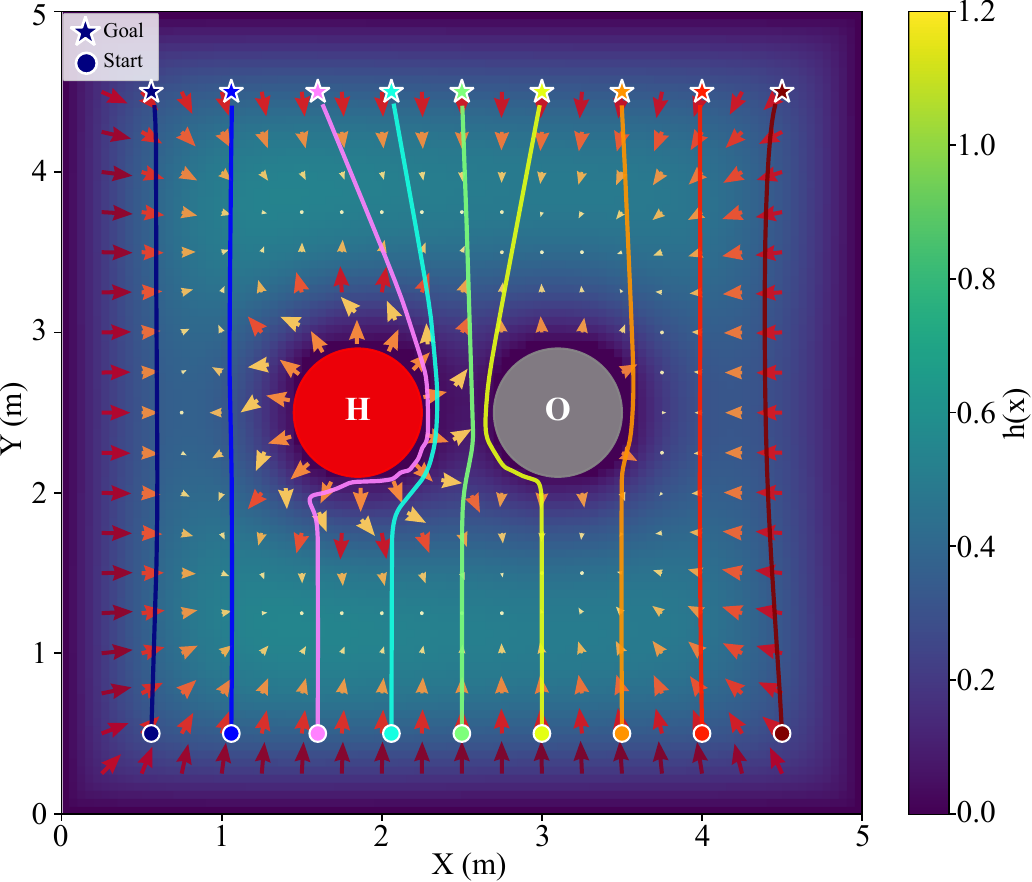}
        \caption{Simulation benchmark of the proposed method safety filtering the robot going to the other side of the arena with a human (H) and a static obstacle (O). It can be observed that the robot would exhibit social compliance unless well away from the human and also keeps a wider berth from the human than the static obstacle. Even where a trajectory passes relatively close to H, it maintains a larger margin than around O (since $b_\text{human} < b_\text{objects}$) and curves with the tangential social bias, so its margin and direction differ from the static-obstacle case.}
        \label{fig:sim_benchmark}    
        % \vspace{-10pt}
    \end{minipage}\hfill
    % Add [t] for top alignment here as well
    \begin{minipage}[t]{0.48\textwidth} 
        \vspace{0pt} % Creates a top-anchor point
        \centering
        \includegraphics[width=\linewidth]{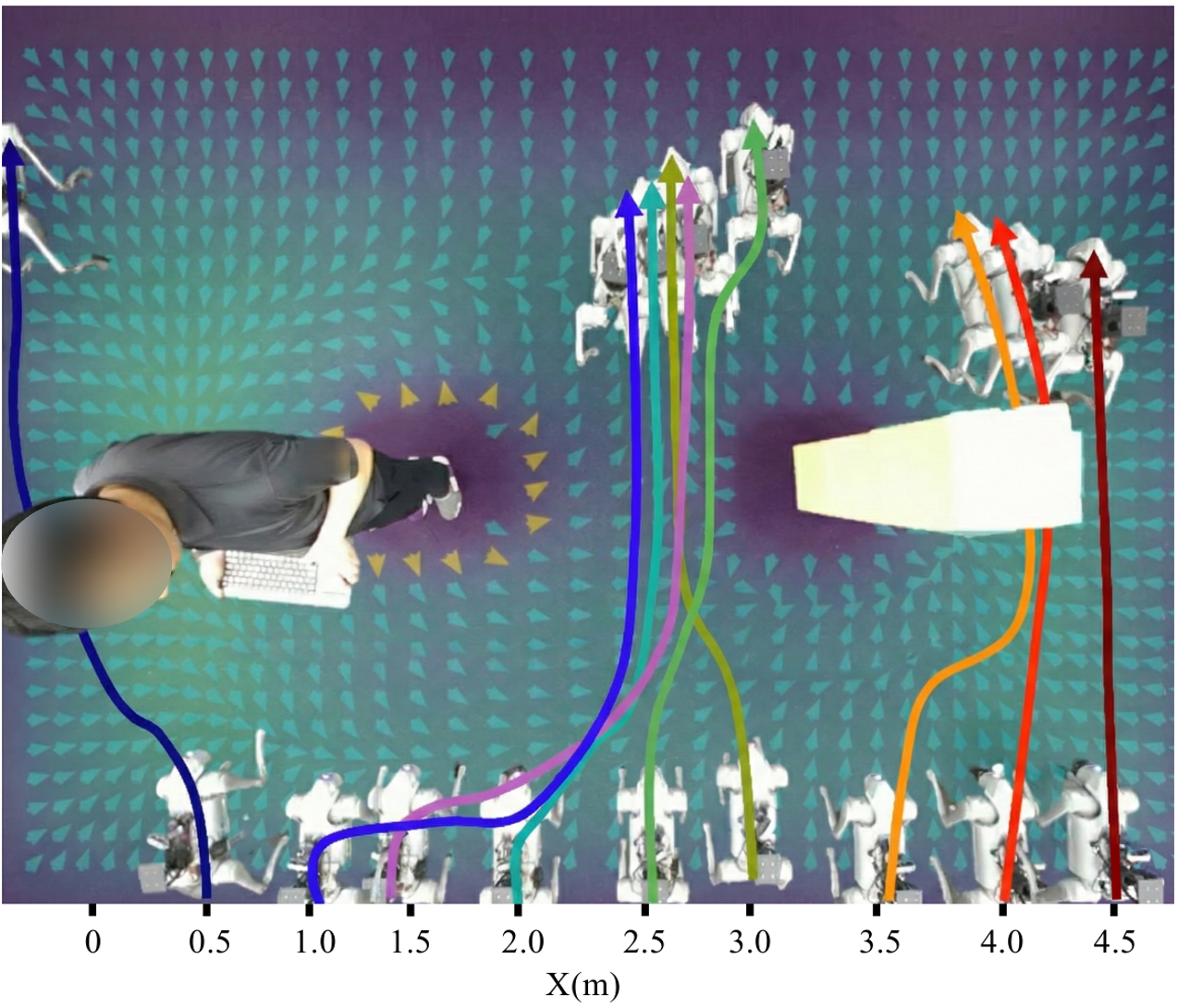}
        \caption{Hardware experiments on the Unitree Go2 quadruped robot. Similar to the simulation benchmark, the robot is tasked with going from one side of the area to the other side. The robot behaves similarly to the simulation, keeping a wider margin to the human and observes social norms whenever possible.}
        \label{fig:social-ablation}    
    \end{minipage}
    \vspace{-6pt}

\end{figure*}
% The theorem demonstrates that although the guidance field $\mathbf{v}$ is Lipschitz continuous—owing to the internal Dirichlet interface $\partial \Omega_r$ encoding social norms—the constraint \eqref{eq: risk aware constraint} is well-defined pointwise, and forward invariance, as presented in \cite{bahati2025risk}, is retained.
%The solution is computed in parallel across all 10 yaw slices using OpenMP.
\vspace{-6pt}
\subsection{Temporal Safety Variation}
% \vspace{-5pt}
We actively estimate a temporal derivative $\partial h_{\text{full}} / \partial t$ to account for dynamic obstacles and perception latency via motion-compensated finite differences:
\begin{equation}
\frac{\partial h_{\text{full}}}{\partial t}(i, j, l) = \frac{h_{\text{full}}^{(k)}(i, j, l) - h_{\text{full}}^{(k-1)}(i + \delta_y, j + \delta_x, l)}{\Delta t_{\text{grid}}}, \nonumber
\end{equation}
where $(\delta_x, \delta_y)$ compensates for robot egomotion between frames and $h_{\text{full}}^{(k-1)}$ and $h_{\text{full}}^{(k)}$ are consecutive safety functions 
% \vspace{-20pt}
solutions separated by time $\Delta t_{\text{grid}}$.
The estimate is smoothed over time using a first-order low-pass filter.

To prevent over-conservatism far from obstacles, we scale $\partial h_{\text{full}} / \partial t$ with the following $\sigma$:
% \vspace{-2pt}
\begin{equation}
\sigma = \min\left(\frac{\|\mathbf{v}_{\text{sem}}\|}{\|\nabla h_{\text{full}}\| + \epsilon \left(1 - e^{-\kappa \max(0, h_{\text{full}})}\right)}, 1\right).
\label{eq:align-fix}
\end{equation}
%with $\epsilon = 0.1$ and $\kappa = 5.0$ (configurable). The scaled temporal term becomes:
This ensures that $\partial h_{\text{full}} / \partial t$ has its full effect near boundaries (where $h_{\text{full}} \approx 0$) but diminishes in open space.
For dynamic obstacles, we rely on a motion-compensated, constant-velocity estimate of $\partial h_{\text{full}} / \partial t$, which is then forward-propagated over the MPC horizon (Sec.~\ref{subsec:mpc}). 
\subsection{MPC Safety Filter}
\label{subsec:mpc}
We construct an MPC optimization problem to plan safe trajectories for the future.
\begin{equation}
\begin{aligned}
\min_{\mathbf{u}_{0:N-1}} \quad & \sum_{k=0}^{N-1} (\mathbf{u}_k - \mathbf{u}_{\text{nom}})^\top P_u (\mathbf{u}_k - \mathbf{u}_{\text{nom}}) \\
\text{subject to} \quad & \boldsymbol{\zeta}_{k+1} = \boldsymbol{\zeta}_k + \Delta t \, \mathbf{u}_k, \\
& v_{x/y} \in [v_{x/y_{\text{min}}}, v_{x/y_{\text{max}}}], \\
& \omega \in [\omega_\textup{min},\omega_\textup{max}], \\
h_{\text{full}}&(\boldsymbol{\zeta}_{k+1}) \geq e^{-\gamma \Delta t} h_{\text{full}}(\boldsymbol{\zeta}_k), \quad k = 0, \ldots, N-1
\end{aligned} \nonumber
\end{equation}
where $\boldsymbol{\zeta}_k = (\mathbf{q}_k, \psi_k)$ is the full state, and the CBF constraint is evaluated using the 2D slice $h(\mathbf{q})$ at each corresponding yaw direction in lieu of the full 3D function $h_{\text{full}}(\boldsymbol{\zeta})$.
We linearize the CBF constraint around the current trajectory using the guidance field $\mathbf{v}$ for the position components and numerical differentiation for the yaw component. 
A Sequential Quadratic Programming (SQP) loop with line search updates the linearization point until the cost residual converges. The time-varying safety function is forward-propagated as $h_{\text{full}}(t_k) = h_{_{\text{full}},0} + \sigma \frac{\partial h_{\text{full}}}{\partial t} t_k$ for each yaw slice, enabling proactive avoidance of closing obstacles.
% The QP is solved via OSQP with warm-starting enabled. If the solver reports primal infeasibility (e.g., due to conflicting constraints near tight corners), it is reset and re-initialized. Both the MPC and real-time filters share the same semantic fields—$h$, $\mathbf{v}$, and $\partial h / \partial t$—ensuring consistent safety behavior across the planning and execution layers.
% \setlength{\textfloatsep}{10pt plus 2pt minus 2pt}
\vspace{-10pt}
\subsection{Realtime Analytical Safety Filter}
For immediate safety enforcement at the state update rate, a closed-form analytical safety filter projects the nominal velocity onto the safe control set.
Given the MPC solution $\mathbf{u}_{\text{mpc}}$ and current state $\boldsymbol{\zeta} = (\mathbf{q}, \psi)$, the safety function and guidance field are evaluated via trilinear interpolation across the 3D grid.
We evaluate the corresponding 2D slice $\psi \mapsto h_\textup{full}(\mb{q}, \psi)$ which yields the semantics-aware safety filter:
% \begin{equation}
% a = \gamma h + \mathbf{v}_\text{sem}^\top \mathbf{u}_{\text{mpc}} + \sigma \frac{\partial h}{\partial t} 
% %+ \frac{\partial h}{\partial \psi} u({\psi})% - \text{ISSf},
% \end{equation}
\begin{align}
\label{eq:cbf-qp-guidance}
    \mb{u}_\textup{safe} = &\argmin_{\mb{u} \in \re^3} \quad \|\mb{u} - \mb{u}_{\text{mpc}}\|_2^2 %\label{eq: risk-aware safety filter}
    \\
    & 
 \mathrm{s.t.} 
 \begin{bmatrix}
      \mathbf{v}_\textup{sem} (\mb{q}, \psi), &  \frac{\partial h_\textup{full}}{\partial \psi}
 \end{bmatrix}\cdot 
  \mb{u} \\
  &\quad \quad  \,+ \frac{\lVert \mathbf{v}_{\textup{sem}}(\mb{q}, \psi)\rVert}{\lVert \nabla  h_\textup{full} \rVert + \sigma(h_\textup{full})} \cdot \frac{\partial{h}_\textup{full}}{\partial{t}}\geq - \gamma h_\textup{full}(\mb{q}, \psi). \nonumber
\end{align}

Although \eqref{eq:cbf-qp-guidance} is class-agnostic in form, the semantic distinction enters upstream: the class-dependent flux $b(\mathbf{q},\psi)$ in \eqref{eq: social laplace} propagates into $\mathbf{v}_\textup{sem}$ and, via the forcing function, into $h_\textup{full}$, so the same filter produces class-dependent margins.

 % \begin{remark}
 %     test
 % \end{remark}
% where ISSf denotes input-to-state safety robustness margins that scale with $\|\mathbf{v}\|$. 
% The safe input is computed using the explicit solution to \eqref{eq:cbf-qp}
% \vspace{-2pt}
% \begin{equation}
% \mathbf{u} = \mathbf{u}_{\text{mpc}} + \frac{-a + \sqrt{a^2 + \beta^2}}{2b} P_u^{-1} \mathbf{v}_\text{sem},
% \end{equation}
% where $\beta = \mathbf{v}_\text{sem}^\top P_u^{-1} \mathbf{v}_\text{sem}$ and $P_u$ weight the inputs.
% whose closed-form expression is:
% % (omitting dependency on $\by$ for brevity):
% %
% % \begin{align}\label{eq: min-norm closed form}
% % \bk_\mathrm{QP} =  \bk_{\mathrm{nom}}+ \frac{\max\{0,-\bvv \cdot \bk_\mathrm{nom} +\gamma h\}}{\|\bvv \|^2}\,\bvv.
% % \end{align}
% %
% \begin{align}\label{eq: k_qp}
%     \bk_{\mathrm{QP}}(\by) = \bk_{\mathrm{nom}}(\by) + \frac{\mathrm{ReLU}(-a(\by))}{\|\bv_\textup{sem}(\by)\|^2}\,\bv_\textup{sem}(\by),
% \end{align}
%We use the ``half-Sontag'' variant (note the factor of 2 in the denominator) because we found full Sontag to be overly aggressive in practice.

\begin{remark}
    We note that the above semantics aware safety filter  modifies $\dot{\psi}$ using $\frac{\partial h_\textup{full}}{\partial \psi}$.
     Due to $\frac{\partial h_\textup{full}}{\partial \psi}$ not being perfectly parallel to the inward pointing normal vector on the boundary of 
    % $\cup_{\psi \in \mathbb{S}} \Omega(\psi)$
        $\tilde{\Omega} = \bigcup_{\psi \in \mathbb{S}^1} \overline{\Omega}   ({\psi}) \times \{\psi\} \subset \mathbb{R}^2 \times \mathbb{S}^1$ \cite{bena2025geometry}
%$\frac{\partial \mathbf{v}_\text{sem}}{\partial \psi}$ 
    , there is a discrepancy in the direction that prevents us from applying Nagumo's theorem. We can accommodate this discrepancy using a robustness term as in the ISSf-CBF literature \cite{kolathaya2018input}. Alternatively, we can scale the  $\frac{\partial h_\textup{full}}{\partial \psi}$ by a similar method to \eqref{eq:align-fix} to maintain CBF-based forward invariance guarantees.    % we can no longer claim forward invariance, to resolve this we can apply a process similar to }%\eqref{}
    We emphasize the scope of this gap: Theorem~\ref{prop:forward_invariance} is a \emph{per-slice} guarantee for each fixed heading $\psi$, so the position-subspace invariance is unaffected and the discrepancy arises only in the coupled heading channel as $\psi$ varies. The two remedies above close this gap, and safety holds across all hardware trials in practice.
\end{remark}
% \subsection{Overall Algorithm}

% Algorithm~\ref{alg:semantic_safety} summarizes the complete semantic safety pipeline. The grid-rate loop (15~Hz) updates the safety function and guidance field, while the state-rate loop (100~Hz) applies the real-time safety filter. The MPC operates asynchronously at 10~Hz.
% Fig. \ref{}

\section{Experiments}
\begin{figure}[t]
    \centering
    \includegraphics[width=0.6\linewidth]{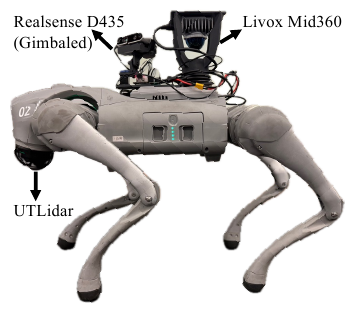}
    \caption{The Unitree Go2 quadruped hardware setup used primarily for experiments. The point clouds from the lidars and the RGB-D camera are aggregated for occupancy map calculation, and the RGB image from the camera is used for semantic understanding.}
    \label{fig:hardware}
    \vspace{-25pt}
\end{figure}
\begin{figure*}[t]
    \centering
    \includegraphics[width=\textwidth]{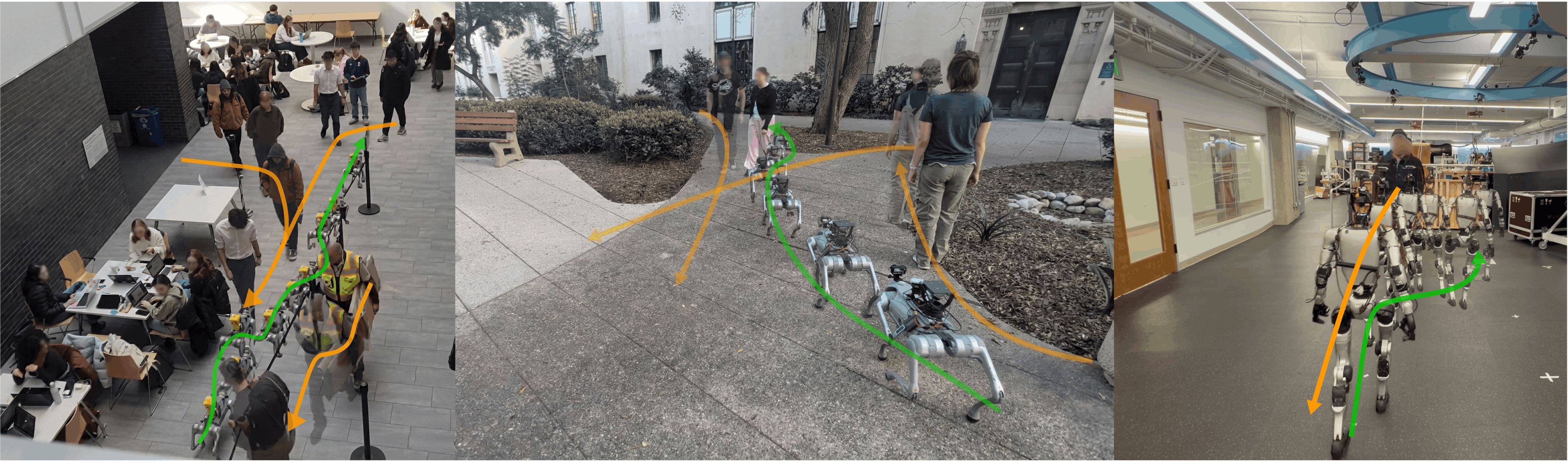}
    \caption{Snapshots of the experiments demonstrating the effectiveness of the proposed method in different environments and on different robots. Left: Safe-SAGE deployed with different hardware setup in cafeteria; middle: Safe-SAGE deployed outside; right: Safe-SAGE deployed on a humanoid robot.}
    \label{fig:snapshots}
    % \vspace{5pt}
\end{figure*}
We assess the social-semantic guidance field modulation in terms of safety and social compliance on the Unitree Go2 quadruped, equipped with a front UTLidar, a top Livox Mid360, and a gimbal-mounted RealSense D435 (Fig.~\ref{fig:hardware}).
The three sensors' point clouds are fused into a robot-centric occupancy grid; the lightweight YOLOv11n \cite{yolo11_ultralytics} performs semantic segmentation on the onboard Jetson Orin NX, and odometry is from FastLIO2 \cite{xu2022fast}.
The Laplace and Poisson equations are solved on a $100\times100$ grid ($0.05$~m, $5\times5$~m) across $10$ yaw slices via a GPU red-black SOR solver. On the onboard Jetson Orin NX, the Poisson solve takes ${\sim}25$--$30$~ms for a total grid loop of ${\sim}50$--$70$~ms (${\sim}15$~Hz); the MPC runs at $100$~Hz and the analytical filter at the odometry rate.

\subsection{Safety and Social Compliance Analysis}
% 1. quantify the sim results with hardware
% 2. success rate w/wo social semantic flux modulation when people corner it to wall

We construct the social-semantic compliance benchmark, requiring the robot to navigate through a gap between a human and a static obstacle, starting at different positions along the $x$-axis.
As can be seen in Fig. \ref{fig:social-ablation}, the behavior of the robot closely matches that of the simulation in Fig. \ref{fig:sim_benchmark}.

Here, the nominal command $\mathbf{u}_\text{nom}$ minimally modified by the filter is the operator's velocity command.
We also perform ablation studies by comparing and quantifying our proposed guidance field modulation with the nominal guidance field, which has no importance differentiation and no tangent bias.
This nominal field is still a Poisson safety filter, so it avoids collisions but lacks the class-dependent margin and social passing our modulation adds.
In this case, we set up a U-shaped corridor with a human pushing the robot towards a wall.
% \begin{figure}
%     \centering
%     \includegraphics[width=0.65\linewidth]{figures/hardware (cropped) (pdfresizer.com).pdf}
%     \caption{The Unitree Go2 robot with a gimbaled Realsense D435, a Livox Mid360 Lidar and an UTLidar}
%     \label{fig:placeholder}
% \end{figure}
During the experiments, the human stops and the free space is enclosed, trapping the robot; thus, the safe controller is forced to stabilize in the closed safe set.
We set $b_{\text{human}}(\mathbf{q})=-1.7$ and $b_{\text{objects}}(\mathbf{q})=-0.5$ for the proposed safety filter while setting $b_{\text{human}}(\mathbf{q})=-1.0$ and $b_{\text{objects}}(\mathbf{q})=-1.0$ for the nominal safety filter, with the resulting margins demonstrated in Fig.~\ref{fig:weight_benchmark}.

For importance differentiation, we measure the human-robot margin (deviation from the center of the enclosed free space, negative toward / positive away from the human); for tangent biasing, we show the maximum lateral offset to pass on the far side of the human.
As can be seen in Table \ref{tab:safety-ablation} our proposed guidance field modulation achieves higher metrics in both categories.
\subsection{Real-World Deployment Scenarios}
% Hallway test
We evaluate Safe-SAGE in three scenarios (Fig.~\ref{fig:head}, Fig.~\ref{fig:snapshots}): a hallway with pedestrians walking toward and away from the robot, an open area with sensory noise, and a crowded cafeteria. Across these, the robot navigates while observing social norms and maintaining safety and social compliance in dynamic, complex environments.
As our method is platform-agnostic and operates on a reduced-order model, it transfers to other robots with minimal changes to sensor inputs and control interfaces; we thus also deploy it on the Unitree G1 humanoid (only a forward-facing D435 with the two lidars), as in Fig.~\ref{fig:snapshots}.
% \vspace{-5pt}
\section{Conclusion}
In this paper, we propose Safe-SAGE, a unified framework that injects social-semantic awareness into robot safety filters.
We extended Poisson safety functions (PSFs) \cite{bahati2025dynamic} and Laplace guidance fields (LGFs) \cite{bahati2025risk} into a semantics-aware safety layer between perception and control, provided a formal analysis of the guidance field modulation, and validated it in simulation and on the Unitree Go2 quadruped and G1 humanoid. The modulation significantly improves the robot's ability to maintain safety and social compliance, distinguishing semantically different obstacles and observing social norms.
Our safety layer operates on a first-order reduced-order model; extension to higher-order systems is a natural next step via a tracking controller on learned full-order dynamics \cite{yang2025shield} or high-order CBF constructions.
Looking forward, we plan to add occupancy-grid memory and semantic graph construction, incorporate large language model reasoning, and target more complex social scenarios.
% \clearpage\newpage

\newsec{Acknowledgments}
The authors used ChatGPT and Claude Code for grammar and editing of the text, for debugging experiment code, and have reviewed and verified all AI-generated content for correctness.
\balance
\bibliographystyle{IEEEtran}
\bibliography{references,Cosner}
% \begin{figure}
%     \centering
%     \includegraphics[width=\linewidth,angle=270]{figures/IMG_7904.jpeg}
%     \caption{Caption}
%     \label{fig:placeholder}
% \end{figure}

\end{document}